\documentclass[conference]{IEEEtran}

\usepackage{graphics} 
\usepackage{epsfig} 
\usepackage{mathptmx} 
\usepackage{times} 
\usepackage{amsmath} 
\usepackage{amssymb}  
\usepackage{epstopdf}
\usepackage{color,colortab}
\usepackage{graphicx}
\usepackage{subcaption}
\usepackage{verbatim}
\usepackage{xcolor}

\newtheorem{thm}{Theorem}
\newtheorem{proof}{proof}
\newtheorem{rem}{Remark}

\title{\LARGE \bf
A variable rest length impedance grasping strategy in the port-Hamiltonian framework}

\begin{document}

\title{A variable rest length impedance grasping strategy in the port-Hamiltonian framework}

\author{
\IEEEauthorblockN{Mauricio Mu\~noz-Arias}
\IEEEauthorblockA{
\textit{Faculty of Science and Engineering}\\
\textit{University of Groningen}\\
Groningen, The Netherlands \\
m.munoz.arias@rug.nl
}
\and
\IEEEauthorblockN{Jacquelien M.A. Scherpen}
\IEEEauthorblockA{
\textit{Faculty of Science and Engineering}\\
\textit{University of Groningen}\\
Groningen, The Netherlands \\
j.m.a.scherpen@rug.nl
}
\and
\IEEEauthorblockN{Alessandro Macchelli}
\IEEEauthorblockA{
\textit{Department of Electrical, Electronic,}\\
\textit{and Information Engineering,}\\
\textit{University of Bologna.}\\
Gviale Risorgimento 2, 40136 Bologna, Italy. \\
m.munoz.arias@rug.nl
}
}

\maketitle

\begin{abstract}
This work is devoted to an impedance grasping strategy for a class
of standard mechanical systems in the port-Hamiltonian framework. We embed a variable rest-length of the springs of the existing impedance grasping strategy in order to achieve a stable non-contact to contact transition, and a desired grasping force. We utilize the port-Hamiltonian structure of standard mechanical systems. 
First, we utilize a change of variables that transforms the port-Hamiltonian 
system into one with constant mass-inertia matrix.
We then achieve impedance grasping control via a \emph{virtual
spring} with a variable rest-length.
The force that is exerted by the virtual spring 
leads to a dissipation term in the impedance grasping controller,
which is needed to obtain a smoother non-contact
to contact transition. 
Simulations and experimental results are given in order to motivate
our results.
\end{abstract}

\section{Introduction}

In order to perform complex robotic tasks involving the interaction
of the end-effector and an external environment, a strategy with dexterous
manipulation is required. 
An impedance grasping strategy represents a suitable solution for mechanical systems 
which means the capability to constrain objects with an end-effector (gripper) 
\cite{Hogan,Siciliano,Spongelal2006}.

A conventional impedance control strategy in the Euler-Lagrange (EL) framework
is a feedback transformation such that the closed-loop system is equivalent
to a mechanical system with a desired behavior, 
\cite{Canudas,Hogan,Spongelal2006}.
More recently in the EL framework, 
the concept of contact estimation in order to improve the classical results of impedance
control is introduced in \cite{Diolaitietal2005}, 
and kinematic redundancy for safe interaction of the robot system with the environment
is given by \cite{Sadeghianetal2012}. 
Passivity-based control in the EL framework 
is based on selecting a storage energy function,
which ensures the desired behavior between the environment
and the mechanical system. 
However, 
the desired storage function does not qualify as an energy function 
in any meaningful physical sense as stated in \cite{Canudas,OrtegaLoriaNicklassonSira}. 
In comparison with the EL framework, 
the port-Hamiltonian (PH) framework has cleaner tuning opportunities, 
resulting in a better performance \cite{Duindam,Fujimoto2001,Arjan}. 
An impedance grasping control approach in the PH framework is given by \cite{Stramigioli2001}. 
This approach introduces the concept of \emph{virtual object},
and its interconnection with the end-effector
and the environment via \emph{configuration springs}.
The grasping of the real object (environment) is obtained via an indirect
control of the position of the virtual object. 
In \cite{Stramigioli2001} actual contact points and measuring of contact forces
are not considered for embedding in an impedance control strategy. 
In addition, an impedance control design methodology in the PH framework 
with Casimir functions is proposed by \cite{Sakai2012},
where the input of the mechanical system is different from the standard case, 
i.e., it is not a torque but a fluid flow.
{\color{blue}
More recently, in \cite{angerer2017port} a variable stiffness coefficient for virtual springs is the key strategy for grasping and manipulation via a port-Hamiltonian framework. 
Furthermore, an energy-balancing passivity-based impedance control strategy is applied to an unmanned aerial vehicle in \cite{rashad2019port} where both motion and interation control is achieved with the key playing role of a virtual spring stiffness.}

The main contribution of this paper is a novel impedance
grasping control strategy for mechanical systems to improve the transient behavior in non-contact to contact transitions. 
We realize a passivity-based control strategy, 
and a type of modified integral control action which can be interpreted as a . 
a virtual spring  with a variable  \emph{rest-length}. 
This means that we can shape the potential energy relative to a grasping force. 
We achieve this by 
a coordinate transformation  to include a \emph{virtual position} error 
in the passive output of the transformed PH system.
The use of the coordinate transformation is inspired by the results of \cite{Dirksz2012}. The main advantage of our results above \cite{Stramigioli2001} and the classical impedance controller of \cite{Hogan} is that our grasping strategy is less destructive in nonrigid environments. This is illustrated in experiments with the Philips Experimental Robotic Arm (PERA), \cite{Rijs}. 
Our grasping strategy can be interpreted as a combination of impedance control and force control,
i.e., impedance control is employed to manage the \emph{transient behaviour} of the grasping, 
%
%
and force control is employed to deal with the steady-state response of the system, 
i.e., we specify how strong (or weak) is the grasp. 

%
%
%

The grasping force depends on the object (environment) dynamics and position. 
We compensate here for nonlinear dynamics of the nonrigid body 
by the feed back of the measurements of our force sensors \cite{Munoz-Arias2013a}. 
In principle, estimation techniques via 
position-based visual servo control or
image-based visual servo control, 
as in \cite{HutchinsonHagerCorke1996}, 
leads to a priori knowledge of the position of the object.
Given this information, we are able to successfully implement our impedance grasping strategy successfully.
Vision control is out of the scope of this paper even though there exists literature about
vision control in the PH setting \cite{Dirksz2012b,Dirksz2014a,Stramigioli2012},
which is promising to connect to our setting in future research.

Preliminary results with simulations of an end-effector of one degree of freedom (DOF) are in \cite{Munoz-Arias2014b}.
Contrary to \cite{Munoz-Arias2014b}, here we present a more clear physical interpretation of our controller, 
we introduce a smoothed potential energy compensation in the model of the end-effector, and we provide experimental results.  
%
%
%
%
%

%

The paper is organized as follows. 
In Section \ref{Grasping_Control_Preliminaries},
we provide a general background in the PH framework, 
especially for a class of standard mechanical systems, 
and we include a brief summary about systems modeling with actuation of additional external forces.
In addition, 
we apply the results of \cite{Viola} to equivalently describe the original PH
system in a PH form which has a constant mass-inertia matrix in the
Hamiltonian via a change of variables. 
Then, 
in Section \ref{Grasping_Strategy} 
we introduce our impedance grasping strategy with the interpretation of virtual spring with a variable rest-length.
%
%
%
We obtain asymptotic stability to a desired grasping force. 
Finally, simulations and experimental results are given in Section 
{\color{red} TBA},
and Section \ref{Grasping_Remarks} provides concluding remarks.


\section{Preliminaries}

\label{Grasping_Control_Preliminaries}

We briefly recap the definition, properties and advantages of modeling
and control with the PH formalism. Then, we recap the results of \cite{Fujimoto2001}
in terms of generalized coordinates transformations for PH systems,
and finally we apply the results of \cite{Viola} to equivalently
describe the original PH system in a PH form which has a constant
mass-inertia matrix in the Hamiltonian.

\subsection{Port-Hamiltonian Systems}

The PH framework is based on the description of systems in terms of
energy variables, their interconnection structure, and power ports.
PH systems include a large family of physical nonlinear systems. The
transfer of energy between the physical system and the environment
is given through energy elements, dissipation elements and power preserving
ports \cite{Duindam,MaschkeArjan}. A time-invariant PH system, introduced by \cite{MaschkeArjan}, is
described by
\begin{equation}
    \Sigma=\begin{cases}
    \begin{array}{c}
    \dot{x}=\left[J\left(x\right)-R\left(x\right)\right]\dfrac{\partial H\left(x\right)}{\partial x}+g\left(x\right)w\\
    \\
    y=g\left(x\right)^{\top}\dfrac{\partial H\left(x\right)}{\partial x}
    \end{array}
    \end{cases}
    \label{eq:Port-Hamiltonian}
\end{equation}
with $x\in\mathbb{R}^{\mathcal{N}}$ the states of the system, the
skew-symmetric interconnection matrix $J\left(x\right)\in\mathbb{R}^{\mathcal{N}\times\mathcal{N}}$,
the symmetric, positive-semidefinite damping matrix $R\left(x\right)\in\mathbb{R}^{\mathcal{N}\times\mathcal{N}}$,
and the Hamiltonian $H\left(x\right)\in\mathbb{R}$. The matrix $g\left(x\right)\in\mathbb{R}^{\mathcal{N}\times\mathcal{M}}$
weights the action of the control inputs $w\in\mathbb{R}^{\mathcal{M}}$
on the system, and $w$, $y\in\mathbb{R}^{\mathcal{M}}$ with $\mathcal{M}\leq\mathcal{N}$,
form a power port pair. We now restrict the analysis to the class
of standard mechanical systems.

Consider a class of standard mechanical systems of $n$-DOF as in \eqref{eq:Port-Hamiltonian}, e.g., an $n$-dof
rigid robot manipulator. Consider furthermore the addition of an external
force vector. The resulting system is then given by
\begin{align}
\left[\begin{array}{c}
	\dot{q}\\
	\dot{p}
\end{array}\right] & =
\left[\begin{array}{cc}
   0_{n\times n} & I_{n\times n}\\
	-I_{n\times n} & -D\left(q,p\right)
\end{array}\right]
\left[\begin{array}{c}
	\dfrac{\partial H\left(q,p\right)}{\partial q}\\
	\dfrac{\partial H\left(q,p\right)}{\partial p}
\end{array}\right]
\nonumber
\\
& +
\left[\begin{array}{c}
	0_{n\times n}\\
	G\left(q\right)
\end{array}\right]u
+
\left[\begin{array}{c}
	0_{n\times n}\\
	B\left(q\right)
\end{array}\right]f_{e}
\label{eq:pre_PH_Mechanical}\\
y & = G\left(q\right)^{\top}\dfrac{\partial H\left(q,p\right)}{\partial p}
\label{eq:pre_PH_Mechanical_y}
\end{align}
with the vector of generalized configuration coordinates $q\in\mathbb{R}^{n}$,
the vector of generalized momenta $p\in\mathbb{R}^{n}$, the identity
matrix $I_{n\times n}$, the damping matrix $D\left(q,p\right)\in\mathbb{R}^{n\times n}$,
$D\left(q,p\right)=D\left(q,p\right)^{\top}\geq0$, $y\in\mathbb{R}^{n}$
the output vector, $u\in\mathbb{R}^{n}$ the input vector, $f_{e}\in\mathbb{R}^{n}$
the vector of external forces, $\mathcal{N}=2n$, matrix $B\left(q\right)\in\mathbb{R}^{n\times n}$,
and the input matrix $G\left(q\right)\in\mathbb{R}^{n\times n}$ everywhere
invertible, i.e., the PH system is \emph{fully actuated}. The Hamiltonian
of the system is equal to the sum of kinetic and potential energy,
\begin{equation}
    H\left(q,p\right)=\dfrac{1}{2}p^{\top}M^{-1}\left(q\right)p+V\left(q\right)
    \label{eq:pre_H}
\end{equation}
where $M\left(q\right)=M^{\top}\left(q\right)>0$ is the $n\times n$
inertia (generalized mass) matrix and $V\left(q\right)$ is the potential
energy.

\begin{rem}
\label{rem:B_Matrix}
The robot dynamics is given in joint space in \eqref{eq:pre_PH_Mechanical}, and here the \emph{external forces} $f_{e}\in \mathbb{R}^n$ are introduced.
Since $f_{e}$ is a vector of external forces, $B\left(q\right)\in \mathbb{R}^{n \times n}$ 
is the transpose of the \emph{geometric Jacobian} \cite{Spongelal2006} that maps the forces in the
\emph{work space} to the (generalized) forces in the \emph{joint space}. 
In this paper the following holds,
\begin{equation}
f_{e} = 
\mathcal{J}\left(q\right)^{\top}F_{e},
\:\:\:\:\: F_{e} \in \mathbb{R}^{N},
\label{eq:pre_Forces_Jacobian}
\end{equation}
and the geometric Jacobian is given by
\begin{equation}
\mathcal{J}\left(q\right)
=
\left[\begin{array}{c}
	\mathcal{J}_{v}\left(q\right)\\
	\mathcal{J}_{\omega}\left(q\right)
\end{array}\right]
\in \mathbb{R}^{6 \times n}
\label{eq:pre_Forces_Jacobian_Linear_Angular}
\end{equation}
where $	\mathcal{J}_{v}\left(q\right)\in\mathbb{R}^{3\times n}$, 
and $\mathcal{J}_{\omega}\left(q\right)\in\mathbb{R}^{3\times n}$ 
are the linear, and angular geometric Jacobians, respectively, and 
$N=\left\lbrace 3,6 \right\rbrace$.
%
If the Jacobian is full rank, 
we can always find $f_{e} \in \mathbb{R}^{n}$ that corresponds to $F_{e}$. 
Then, it is not a limitation to suppose $B\left(q\right)=I_{n}$. 
This separation between joint and work spaces is important here,
because we control the robot by acting on the generalized coordinates $q$,
i.e., in the joint space, but we grasp objects with the end-effector in the work space. 
\end{rem}

We consider the PH system \eqref{eq:pre_PH_Mechanical} as a class of
standard mechanical systems with external forces.

\subsection{Nonconstant to constant mass-inertia matrix transformation}

Consider a class of standard mechanical systems in the PH framework
with a nonconstant mass-inertia matrix $M\left(q\right)$ as in \eqref{eq:pre_PH_Mechanical}.
The aim of this section is to transform the original system \eqref{eq:pre_PH_Mechanical}
into a PH formulation with a constant mass-inertia matrix via a generalized
canonical transformation \cite{Fujimoto2001}. The proposed change
of variables to deal with a nonconstant mass inertia matrix is first
proposed by \cite{Viola}.

Consider system \eqref{eq:pre_PH_Mechanical} with nonconstant $M\left(q\right)$,
and a coordinate transformation $\bar{x}=\Phi\left(x\right)=\Phi\left(q,p\right)$
as
\begin{equation}
    \bar{x}=\left(\begin{array}{c}
    \bar{q}\\
    \bar{p}
    \end{array}\right)=\left(\begin{array}{c}
    q-q_{f}\\
    T\left(q\right)^{-1}p
    \end{array}\right)=\left(\begin{array}{c}
    q-q_{f}\\
    T\left(q\right)^{\top}\dot{q}
    \end{array}\right)
    \label{eq:pre_Phi}
\end{equation}
with a constant \emph{virtual desired position} $q_{f}$ that we define
later on, and where $T\left(q\right)$ is a lower triangular matrix such that 
\begin{equation}
T\left(q\right) 
							=
							T\left(\Phi^{-1}\left(q,p\right)\right)
							=
							\bar{T}\left(\bar{q}\right)								
\label{eq:pre_T_Mapping}
\end{equation}
and
\begin{equation}
    M\left(q\right)=T\left(q\right)T\left(q\right)^{\top}=T\left(\bar{q}\right)T\left(\bar{q}\right)^{\top}
    \label{eq:pre_Mass_decomp}
\end{equation}
Consider now the Hamiltonian $H\left(q,p\right)$ as in \eqref{eq:pre_H},
and using \eqref{eq:pre_Phi}, 
we realize $\bar{H}\left(\bar{x}\right)=H\left(\Phi^{-1}\left(\bar{x}\right)\right)$
and $\bar{V}\left(\bar{q}\right)=V\left(\Phi^{-1}\left(\bar{q}\right)\right)$
as
\begin{equation}
    \bar{H}\left(\bar{x}\right)
    						  =
    						  \dfrac{1}{2}
    						  \bar{p}^{\top}\bar{p}
    						  +
    						  \bar{V}\left(\bar{q}\right)
    \label{eq:pre_H_bar}
\end{equation}
The new form of the interconnection and dissipation matrices of the
PH system are realized via the coordinate transformation \eqref{eq:pre_Phi},
the mass-inertia matrix decomposition \eqref{eq:pre_Mass_decomp},
and the new Hamiltonian \eqref{eq:pre_H_bar}.

Consider the system \eqref{eq:pre_PH_Mechanical}, the change of variables
$\Phi\left(q,p\right)$ as in \eqref{eq:pre_Phi}, the $M\left(q\right)$
decomposition as in \eqref{eq:pre_Mass_decomp}, and the Hamiltonian
$\bar{H}\left(\bar{x}\right)$ as in \eqref{eq:pre_H_bar}.
The resulting forced \cite{Arjan} PH system is then given by
\begin{align}
\left[\begin{array}{c}
	\dot{\bar{q}}\\
	\dot{\bar{p}}
\end{array}\right] 
& =
\left[\begin{array}{cc}
	0_{n\times n} & T^{-\top}\\
 -T^{-1} & \bar{J}_{2}-\bar{D}
\end{array}\right]
\left[\begin{array}{c}
	\dfrac{\partial\bar{H}\left(\bar{q},\bar{p}\right)}{\partial\bar{q}}\\
	\dfrac{\partial\bar{H}\left(\bar{q},\bar{p}\right)}{\partial\bar{p}}
\end{array}\right]
\nonumber\\
& 
+
\left[\begin{array}{c}
	0_{n\times n}\\
	\bar{G}\left(\bar{q}\right)
\end{array}\right]v
+
\left[\begin{array}{c}
	0_{n\times n}\\
	\bar{B}\left(\bar{q}\right)
\end{array}\right]f_{e}
\label{eq:pre_PH_Error_System}\\
\bar{y} 
& =\bar{G}\left(\bar{q}\right)^{\top}\dfrac{\partial\bar{H}\left(\bar{q},\bar{p}\right)}{\partial\bar{p}}
\label{eq:pre_PH_Error_System_y}
\end{align}
with a new input $v\in\mathbb{R}^{n}$. The arguments of
$T\left(\bar{q}\right)$,
$\bar{J}_{2}\left(\bar{q},\bar{p}\right)$, and
$D\left(\bar{q},\bar{p}\right)$
are left out for notational simplicity.
The skew-symmetric matrix $\bar{J}_{2}\left(\bar{q},\bar{p}\right)$ takes the form
\begin{equation}
				\bar{J}_{2}\left(\bar{q},\bar{p}\right)
				=
				\dfrac{\partial\left(\bar{T}^{-1}\bar{p}\right)}
				{\partial \bar{q}}\bar{T}^{-\top}
				-\bar{T}^{-1}
				\dfrac{\partial\left(\bar{T}^{-1}\bar{p}\right)}
				{\partial \bar{q}}^{\top}
				\label{eq:pre_J_2}
\end{equation}
with
\begin{equation}
    \left(q,p\right)=\Phi^{-1}\left(\bar{q},\bar{p}\right)
    \label{eq:Impedance_Phi_Inv}
\end{equation}
together with the matrix $\bar{D}\left(\bar{q},\bar{p}\right)\geq0$,
and the input matrices $\bar{G}\left(\bar{q}\right)$, and $\bar{B}\left(\bar{q}\right)$,
are described by
\begin{eqnarray}
    \bar{D}\left(\bar{q},\bar{p}\right) &=& T\left(\bar{q}\right)^{-1}D\left(\Phi^{-1}\left(\bar{q},    \bar{p}\right)\right)T\left(\bar{q}\right)^{-\top}
    \label{eq:Impedance_D_bar}\\
%
    \bar{G}\left(\bar{q}\right) &=& T\left(\bar{q}\right)^{-1}G\left(\bar{q}\right)
    \label{eq:Impedance_G_bar}\\
%
    \bar{B}\left(\bar{q}\right) &=& T\left(\bar{q}\right)^{-1}B\left(\bar{q}\right)
    \label{eq:Impedance_B_bar}
\end{eqnarray}
respectively. Via the transformation \eqref{eq:pre_Phi}, we
then obtain a class of mechanical systems with a constant (identity)
mass inertia matrix in the Hamiltonian function as in \eqref{eq:pre_H_bar},
which equivalently describes the original system \eqref{eq:pre_PH_Mechanical}
with nonconstant mass-inertia matrix. We use the results for our impedance
grasping strategy in the next section.

\section{An Impedance Grasping Strategy}
\label{Grasping_Strategy}

In this section, 
a control law strategy is introduced in order to
achieve an impedance grasping interaction between a mechanical system
and its environment, in a noncontact to contact transition. 
%
Here, we combine two strategies, i.e., impedance control and force control. 
First, impedance control is used to manage transient behavior of the grasping, 
i.e., the interaction between an end-effector and the environment (object).
We improve the response of the system during noncontact to contact transitions
in comparison to former impedance control methods such as \cite{Hogan,Sadeghianetal2012,Sakai2012}. 
Secondly, force control is employed to deal with the steady-state response of the system.
%
%
The problem of stabilization is to find a control law which brings the grasping
force to a desired force $f_{d}$. 
In order to avoid steady-state errors, 
we include dynamics in such a way that the PH structure is preserved. 
Then, via a change of variables for the canonical momenta of system \eqref{eq:pre_PH_Error_System}, 
we realize a passive output in the transformed system that includes the grasping error.
The key idea implemented here lies in the virtual potential energy shaping. 
The virtual potential energy is represented as virtual spring-stored energy,
which rest-length can be varied. 
This means that we can shape the minimum potential energy relative to a grasping force.
Then, 
when the system experiences the noncontact to contact transition,
we obtain asymptotic stability to a desired force which is related to a virtual desired position $q_{f}$. 
The results here are inspired by \cite{Dirksz2012,Stramigioli2001}.

{\color{blue}
We define a\emph{ virtual spring} with a variable rest-length $q_{rl}$.
The force that is exerted by the virtual spring leads to a dissipation
term in the impedance grasping controller,
which is needed to obtain a smoother noncontact to contact transition. 
Of importance here is to make the dynamics of the rest-length dependent on the port output of the system. 
Then, 
the incorporation of a virtual spring force with a variable rest-length fundamentally improves 
mechanical impedance between the mechanical system and the environment. 
In order to implement the virtual spring force we define a \emph{virtual potential energy}
%
%
$\bar{U}\left(\bar{q},\bar{p},q_{rl}\right)$
as
	\begin{align}
	\bar{U}\left(\bar{q},\bar{p},q_{rl}\right)
	= 
	& \bar{p}^{\top}T\left(\bar{q}\right)^{\top}K_{p}\left(\bar{q}-q_{rl}\right)
	\nonumber 
	\\
	& +\dfrac{1}{2}\left(\bar{q}-q_{rl}\right)^{\top}K_{p}T\left(\bar{q}\right)T\left(\bar{q}\right)^{\top}K_{p}\left(\bar{q}-q_{rl}\right)\nonumber \\
	& +\dfrac{1}{2}\left(\bar{q}-q_{rl}\right)^{\top}K_{p}\left(\bar{q}-q_{rl}\right)+\dfrac{1}{2}q_{rl}^{\top}K_{rl}q_{rl}
	\label{eq:Impedance_Virtual_Potential}
	\end{align}
with matrices $K_{p}>0$, $K_{rl}>0$. 
Furthermore, a desired grasping force $f_{d}$ is related to a virtual desired position $q_{f}$, 
a virtual potential energy $\bar{U}\left(\bar{q},\bar{p},q_{rl}\right)$ as in \eqref{eq:Impedance_Virtual_Potential},
a rest-length $q_{rl}$, and a generalized coordinate $q$. 
Based now on the classical concept of impedance control introduced by \cite{Hogan}, we design a grasping force 
$f_{d}$ given when $\left(q,\: q_{rl}\right)$ are asymptotically stabilized to zero,
i.e.,
\begin{align}
f_{d} 
 	& 
 	= K_{p}q_{f}
 	\label{eq:Impedance_desired_force}
\end{align}
%
%

}
\begin{rem}
\label{rem:Desired_Impedance}
The virtual potential energy \eqref{eq:Impedance_Virtual_Potential} is defined in the joint space, 
but the idea is to apply a desired vector force $F_{d} \in \mathbb{R}^{n}$
in the work space in steady state. 
Then, the meaning of \eqref{eq:Impedance_desired_force} is that it is necessary to find
$f_{d} \in \mathbb{R}^{n}$, $K_{d} \in \mathbb{R}^{n \times n}$, and $q_{f} \in \mathbb{R}^{n \times n}$
such that
\begin{align}
f_{d}
&
=
\mathcal{J}^{\top}\left(q_{f}\right)F_{d} = K_{p}q_{f}.
\label{eq:Impedance_desired_force_Remark}
\end{align}
$K_{p}>0$ is interpreted as a desired elastic behavior in the joint space, 
and $q_{f}$ is the reference position in steady state. 
When the dynamics of $q_{rl}$ is given, 
we basically have a desired impedance, i.e., we specify the way in which robot and object interact. 
Hence, the desired impedance is defined in joint space.
\end{rem}

In order to incorporate the variable rest-length in the port output of the system, 
we realize a coordinate transformation
	\begin{equation}
	\hat{p}=
		\bar{p}
		+
		T\left(\bar{q}\right)^{\top}
		K_{p}
		\left(\bar{q}-q_{rl}\right)
	\label{eq:Impedance_p_hat}
	\end{equation}
which then implies 
	\begin{equation}
	\dot{\hat{p}}=
		\dot{\bar{p}}
		+T\left(\bar{q}\right)^{\top}
		 K_{p}
		 \left(\dot{\bar{q}}-\dot{q}_{rl}\right)
		+\dot{T}\left(\bar{q}\right)^{\top}
         K_{p}
		 \left(\bar{q}-q_{rl}\right)
	\label{eq:Impedance_p_hat_dot}
	\end{equation}
The new output becomes
\begin{equation}
 \hat{y}
	=
		\bar{G}
 		\left(\bar{q}\right)^{\top}
		\hat{p}
	=
		\bar{G}
		\left(\bar{q}\right)^{\top}
		\left(
			\bar{p}
			+T\left(\bar{q}\right)^{\top}
			K_{p}
			\left(\bar{q}-q_{rl}\right)
		\right)
	\label{eq:Impedance_y_hat}
	\end{equation}
and finally the dynamics of the rest-length is chosen as a modified integrator,
i.e.,
	\begin{equation}
		\dot{q}_{rl}
		=
			-\hat{y}
			-K_{p}
			 \left(\bar{q}-q_{rl}\right)
			-K_{rl}q_{rl}
	\label{eq:Impedance_q_rest_length_dot}
	\end{equation}
with a constant matrix $K_{rl}>0$.
%

\begin{rem}
It can be seen that a new port-pair 
$\left(u_{rl},y_{rl}\right)$ is now given by the following dynamics
\begin{align}
u_{rl} 
	=
	& 
	\dot{q}_{rl}
	\label{eq:Impedance_u_rl}
	\\
y_{rl}
	=
	&
	\bar{G}\left(\bar{q}\right)^{\top} 
	\dfrac{\partial \left(\bar{H}\left(\bar{q},\bar{p}\right) + \bar{U}\left(\bar{q},\bar{p},q_{rl}\right)\right)}
	 {\partial q_{rl}}
	\nonumber
	\\
	=
	&
	\bar{G}\left(\bar{q}\right)^{\top} 
    \left(  
             K_{p}
			 \left(\bar{q}-q_{rl}\right)
			+K_{rl}q_{rl}
    \right)
	\label{eq:Impedance_y_rl}
\end{align}
with $\bar{H}\left(\bar{q},\bar{p}\right)$ as in \eqref{eq:pre_H_bar}, 
$\bar{U}\left(\bar{q},\bar{p},q_{rl}\right)$ as in \eqref{eq:Impedance_Virtual_Potential},
and the dynamics of $q_{rl}$ as in \eqref{eq:Impedance_q_rest_length_dot}.
%
\end{rem}

We now define an impedance grasping control law of the PH system \eqref{eq:pre_PH_Mechanical}
with measurable external forces, i.e.,
\begin{thm}
\label{Force_Control_Th_T(q)}
Consider the port-Hamiltonian system \eqref{eq:pre_PH_Error_System}
with $\bar{D}\left(\bar{q},\bar{p}\right)$,  
constant matrices $K_{p}>0$ and $K_{rl}>0$, 
invertible matrices $\bar{G}\left(\bar{q}\right)$ and $\bar{B}\left(\bar{q}\right)$, 
and that we have information of the vector of external forces $f_{e}$ via force sensors. 
Consider furthermore a passive output $\hat{y}$ as in \eqref{eq:Impedance_y_hat}, 
and assume that the system is zero-state detectable with respect to $\bar{x}$. 
Then, the control input
\begin{align}
v  
  = 
  & 
  \bar{G}^{-1}
  \left[
     T^{-1}
     \dfrac{\partial\bar{H}\left(\bar{q},\bar{p}\right)}
     {\partial\bar{q}}
    +\bar{G}
     K_{rl}
     q_{rl}
    -\bar{B}
     f_{e}
   \right.
   \nonumber 
   \\
   & 
    +
    \left(
       \bar{G}
       T^{-\top}
      -T^{-\top}
      +\bar{J}_{2}
      -\bar{D}
    \right)
    T^{\top}
    K_{p}
    \left(\bar{q}-q_{rl}\right)
   \nonumber 
   \\
   & 
    \left.
     -T^{\top}
      K_{p}
 	  \left(\dot{\bar{q}}-\dot{q}_{rl}\right)
 	 -\dot{T}^{\top}
 	  K_{p}
 	  \left(\bar{q}-q_{rl}\right)
 	 \right] 
 	-C\hat{y}
 \label{eq:Impedance_Control_Law}
 \end{align}
with $C>0$, 
asymptotically stabilizes the system \eqref{eq:pre_PH_Error_System}
with zero steady-state error at $\bar{q}^{*}=0$.
We have left out the arguments of 
$\bar{G}\left(\bar{q}\right)$,
$\bar{J}_{2}\left(\bar{q},\hat{p}\right)$,
$\bar{D}\left(\bar{q},\hat{p}\right)$,
$T\left(\bar{q}\right)$, and
$\bar{B}\left(\bar{q}\right)$
for notational simplicity.
\end{thm}
\begin{proof}
The coordinate transformation $\bar{x}$ as in \eqref{eq:pre_Phi}
results in $\dot{\bar{q}}$ as 
	\begin{equation}
	\dot{\bar{q}}=\dot{q}=M\left(q\right)^{-1}p=T\left(\bar{q}\right)^{-\top}\bar{p}
	\label{eq:Impedance_Appendix_q_bar_dot-1}
	\end{equation}
Based on the adapted momenta $\hat{p}$ as in \eqref{eq:Impedance_p_hat},
we rewrite the dynamics $\dot{\bar{q}}$ as in \eqref{eq:Impedance_Appendix_q_bar_dot-1}
in terms of $\left(\bar{q},\hat{p},q_{rl}\right)$, 
i.e.,
	\begin{equation}
	\dot{\bar{q}}=-K_{p}\left(\bar{q}-q_{rl}\right)+T\left(\bar{q}\right)^{-\top}\hat{p}
	\label{eq:Impedance_Appendix_q_bar_dot}
	\end{equation}
We differentiate both sides of the change of variables \eqref{eq:Impedance_p_hat}
as
	\begin{equation}
	\dot{\hat{p}}=\dot{\bar{p}}+T\left(\bar{q}\right)^{\top}K_{p}\left(\dot{\bar{q}}-\dot{q}_{rl}\right)+\dot{T}\left(\bar{q}\right)^{\top}K_{p}\left(\bar{q}-q_{rl}\right)
	\label{eq:Impedance_Appendix_p_hat_dot-1}
	\end{equation}
and with the dynamics of $\dot{\bar{p}}$ as in \eqref{eq:pre_PH_Error_System} as 
\begin{align}
\dot{\bar{p}}
		&
		=
	   -T^{-1}
     	\dfrac{\partial\bar{H}}
     	{\partial\bar{q}}
     	+
     	\left(
		  \bar{J}_{2}
         -\bar{D}
        \right)
       	\dfrac{\partial\bar{H}}
     	{\partial\bar{p}}
        +
        \bar{G}
        v
        +
        \bar{B}
        f_{e}
        \nonumber
        \\
        &
        =
       -T^{-1}
     	\dfrac{\partial\bar{H}}
     	{\partial\bar{q}}
     	+
        \left(
		  \bar{J}_{2}
         -\bar{D}
        \right)
        \bar{p}
       	+
        \bar{G}
        v
        +
        \bar{B}
        f_{e}
\label{eq:Impedance_p_bar_dot}
\end{align}
In \eqref{eq:Impedance_p_bar_dot}, we have left out the arguments of
$T\left(\bar{q}\right)$,
$\bar{G}\left(\bar{q}\right)$,
$\bar{H}\left(\bar{q},\bar{p}\right)$, 
$\bar{D}\left(\bar{q},\bar{p}\right)$, and
$\bar{J}_{2}\left(\bar{q},\bar{p}\right)$
for notational simplicity.
We substitute the dynamics of $\dot{\bar{p}}$ as in \eqref{eq:Impedance_p_bar_dot}, and the control law $v$ as in \eqref{eq:Impedance_Control_Law} in \eqref{eq:Impedance_Appendix_p_hat_dot-1}. 
It leads to the dynamics $\dot{\hat{p}}$ in terms of $\left(\bar{q},\hat{p},q_{rl}\right)$, i.e.,
	\begin{align}
	\dot{\hat{p}} = 
		& 
		\left(
			 \bar{J}_{2}\left(\bar{q},\bar{p}\right)
	   		-\bar{D}\left(\bar{q},\bar{p}\right)
	        -\bar{G}\left(\bar{q}\right)
	        C
	        \bar{G}
	        \left(\bar{q}\right)^{\top}
	    \right)
	    \hat{p}
	    \nonumber 
	    \\
	    & 
	    -K_{p}
	    \left(\bar{q}-q_{rl}\right)
	    +
	    \bar{G}
	    \left(\bar{q}\right)
	    \left(
	       K_{p}
	       \left(\bar{q}-q_{rl}\right)
	       +K_{rl}
	       q_{rl}
	    \right)
	\label{eq:Impedance_Appendix_p_hat_dot_substituted}
	\end{align}
with $\hat{p}$ as in \eqref{eq:Impedance_p_hat}, and $\hat{y}$
as in \eqref{eq:Impedance_y_hat}. 
Furthermore, the dynamics of the
variable rest-length $\dot{q}_{rl}$ as in \eqref{eq:Impedance_q_rest_length_dot}
can be rewritten as
	\begin{equation}
	\dot{q}_{rl}=-\bar{G}\left(\bar{q}\right)^{\top}\hat{p}-K_{p}\left(\bar{q}-q_{rl}\right)-K_{rl}q_{rl}
	\label{eq:Impedance_Appendix_q_rl_dot}
	\end{equation}
Finally, we choose a smooth function 
$\bar{U}\left(\bar{q},\bar{p},q_{rl}\right)$ as in \eqref{eq:Impedance_Virtual_Potential},
We then realize a candidate Lyapunov function   $\hat{H}\left(\bar{q},\hat{p},q_{rl}\right)=\bar{H}\left(\bar{q},\bar{p}\right)+\bar{U}\left(\bar{q},\bar{p},q_{rl}\right)$, s.t., $\hat{H}\left(\bar{q},\hat{p},q_{rl}\right)>0$, and with $\bar{H}\left(\bar{q},\bar{p}\right)$ as in \eqref{eq:pre_H_bar}, and the change of variables $\bar{p}$ as in \eqref{eq:Impedance_p_hat}, i.e,
%
%
	\begin{align}
	\hat{H}\left(\bar{q},\hat{p},q_{rl}\right)
	=
	&
	\dfrac{1}{2}\hat{p}^{\top}\hat{p}
	+
	\dfrac{1}{2}\left(\bar{q}-q_{rl}\right)^{\top}K_{p}\left(\bar{q}-q_{rl}\right)
	\nonumber
	\\
	+
	&
	\dfrac{1}{2}q_{rl}^{\top}K_{rl}q_{rl}
	\label{eq:Impedance_Appendix_H_hat}
	\end{align}
Based now on the dynamics $\dot{\bar{q}}$ as in \eqref{eq:Impedance_Appendix_q_bar_dot},
$\dot{\hat{p}}$ as in \eqref{eq:Impedance_Appendix_p_hat_dot_substituted},
and $\dot{q}_{rl}$ as in \eqref{eq:Impedance_Appendix_q_rl_dot}, we obtain the closed-loop
	\begin{equation}
	\left[\begin{array}{c}
	\dot{\bar{q}}\\
	\dot{\hat{p}}\\
	\dot{q}_{rl}
	\end{array}\right]=\left[\begin{array}{ccc}
	-I_{n\times n} & T^{-\top} & 0_{n\times n}\\
	-T^{-1} & -\tilde{D} & \bar{G}\\
	0_{n\times n} & -\bar{G}^{\top} & -I_{n\times n}
	\end{array}\right]\left[\begin{array}{c}
	\dfrac{\partial\hat{H}}{\partial\bar{q}}\\
	\dfrac{\partial\hat{H}}{\partial\hat{p}}\\
	\dfrac{\partial\hat{H}}{\partial q_{rl}}
	\end{array}\right]
	\label{eq:Impedance_Appendix_pH_Closed_Loop}
	\end{equation}
with Hamiltonian \eqref{eq:Impedance_Appendix_H_hat}, where the matrix $\tilde{D}\left(\bar{q},\hat{p}\right)$ is given by
\begin{align}
\tilde{D}\left(\bar{q},\hat{p}\right)
=
&
-\bar{J}_{2}\left(\bar{q},\hat{p}\right)
+
\bar{D}\left(\bar{q},\hat{p}\right)
+
\bar{G}\left(\bar{q}\right)C\bar{G}\left(\bar{q}\right)^{\top}
\label{Impedance_D_tilde}
\end{align}
and where the arguments of $T\left(\bar{q}\right)$, $\tilde{D}\left(\bar{q},\hat{p}\right)$, and $\hat{H}\left(\bar{q},\hat{p},q_{rl}\right)$ are left out for simplicity. 

Take now \eqref{eq:Impedance_Appendix_H_hat} as a candidate Lyapunov
function, $\hat{H}\left(\bar{q},\hat{p},q_{rl}\right)>0$. It can
be verified via the dynamics of $\bar{q}$, $\hat{p}$ and $q_{rl}$,
as in \eqref{eq:Impedance_Appendix_q_bar_dot}, \eqref{eq:Impedance_Appendix_p_hat_dot_substituted},
and \eqref{eq:Impedance_Appendix_q_rl_dot}, respectively, that
$\left(\bar{q},\hat{p},q_{rl}\right)=\left(0,0,0\right)$ is an equilibrium
point of \eqref{eq:Impedance_Appendix_pH_Closed_Loop}. We now compute
the power balance $\dot{\hat{H}}\left(\bar{q},\hat{p},q_{rl}\right)$
as
	\begin{equation}
	\dot{\hat{H}}\left(\bar{q},\hat{p},q_{rl}\right)=-\left[\begin{array}{ccc}
	\hat{q}^{\top} & \hat{p}^{\top} & \hat{q}_{rl}^{\top}\end{array}\right]^{\top}U\left[\begin{array}{c}
	\hat{q}\\
	\hat{p}\\
	\hat{q}_{rl}
	\end{array}\right]
	\label{eq:Impedance_Appendix_H_hat_dot_matricial}
	\end{equation}
with $\hat{q}=\bar{q}-q_{rl}$, $\hat{q}_{rl}=K_{p}\left(\bar{q}-q_{rl}\right)+K_{rl}q_{rl}$,
and a matrix $U$, s.t., 
	\begin{equation}
	U=\left[\begin{array}{ccc}
	K_{p}K_{p} & 0_{n\times n} & 0_{n\times n}\\
	0_{n\times n} & \bar{D}\left(\bar{q},\hat{p}\right)+\bar{G}\left(\bar{q}\right)C\bar{G}\left(\bar{q}\right)^{\top} & 0_{n\times n}\\
	0_{n\times n} & 0_{n\times n} & I_{n\times n}
	\end{array}\right]
	\label{eq:Impedance_Appendix_U_positive}
	\end{equation}
Since $\bar{G}\left(\bar{q}\right)$ is full rank, $\bar{D}\left(\bar{q},\hat{p}\right)\geq0$, and
$C$, $K_{p}$ and $K_{rl}$ are positive definite, then $U>0$, and
thus $\dot{\hat{H}}\left(\bar{q},\hat{p},q_{rl}\right)\leq0$. Hence, since the system is zero-state detectable (see \cite{Fujimoto2001}), then
the closed-loop system \eqref{eq:Impedance_Appendix_pH_Closed_Loop}
is asymptotically stable in $\left(\bar{q},\hat{p},q_{rl}\right)=\left(0,0,0\right)$,
and hence $\bar{q}^{*}=0$. \hfill$\Box$
\end{proof}

\begin{rem}
Since the new output $\hat{y}$ as in \eqref{eq:Impedance_y_hat}
includes a position error $\bar{q}$ and a variable rest-length $q_{rl}$,
we have realized here an additional (co)dissipation term $K_{p}K_{p}>0$
in our power balance $\dot{\hat{H}}\left(\bar{q},\hat{p},q_{rl}\right)$
as in \eqref{eq:Impedance_Appendix_H_hat_dot_matricial}. This additional dissipation
term realized by our impedance strategy \eqref{eq:Impedance_Control_Law}
leads to a smoother noncontact to contact transition during the grasping.
\end{rem}

Summarizing, we have realized an impedance grasping control law via passivity-based
control, and damping injection. Via the control law \eqref{eq:Impedance_Control_Law}
we are able to stabilize the system \eqref{eq:pre_PH_Error_System}
to a virtual desired position $q_{f}$ which means a realization of
a grasping force $f_{d}$ as in \eqref{eq:Impedance_desired_force}
in a noncontact to contact transition. We assume here that the end-effector
is within a grasping distance with respect to the environment \cite{Murray}.

\section{Concluding Remarks}
\label{Grasping_Remarks}

This paper is devoted to the development of a new strategy of impedance
grasping control in the PH framework in a noncontact to contact transition.
Our main motivation is given by the proposition of an alternative
to the classical impedance control methods in the EL and
the PH formalism. 
We have given an impedance control law that consist
of force control complemented with a virtual spring force. 
The incorporation of a virtual spring force with a variable rest-length that can be
varied fundamentally improves the mechanical impedance between the
system and the environment. 
The impedance control law achieves asymptotic stability in the closed-loop system with a zero steady-state error.
Future work includes experimental results based on the proposed impedance
grasping strategy and the implementation of estimation techniques
in order to obtain a grasping force. 
The estimation techniques that we have considered are
based on force feedback during a noncontact to contact transition,
and an estimation of the position of the object. 
Including vision control for position estimation is a topic of future research.



%
  %

%


\begin{thebibliography}{22}

\bibitem{Canudas}C. Canudas de Wit, B. Siciliano and G. Bastin,
\emph{Theory of Robot Control}, London, UK: Springer, 1996.

\bibitem{Dirksz2012}D.A. Dirksz and J.M.A. Scherpen, ``Power-based
control: canonical coordinate transformations, integral and adaptive
control'', \emph{Automatica}, vol. 48, no. 6, pp. 1046-1056, 2012.

\bibitem{Dirksz2012b}D.A. Dirksz amd J.M.A. Scherpen, ``A port-Hamiltonian
approach to visual servo control of a pick and place system'', in \emph{Proc.
51st Conf. Decision and Control}, Hawai, USA, 2012, pp. 5661-5666.

\bibitem{Dirksz2014a}D.A. Dirksz, J.M.A. Scherpen and  M. Steinbuch, ``A port-Hamiltonian approach to visual servo control of a pick and place system'', \emph{Asian Journal of Control}, vol. 16, no. 3, pp. 703-713, 2013.

\bibitem{Diolaitietal2005}N. Diolaiti, C. Melchiorri and S. Stramigioli, ``Contact impedance estimation for robotic systems'', \emph{IEEE
Transactions on Robotics}, vol. 21, no. 5, pp. 925-935, 2005.

\bibitem{Duindam}V. Duindam, A. Macchelli, S. Stramigioli and H.
Bruyninckx (Eds.), \emph{Modeling and Control of Complex Physical
Systems: The Port-Hamiltonian Approach}, Berlin, Germany: Springer-Verlag, 2009.

\bibitem{Fujimoto2001}K. Fujimoto and T. Sugie, ``Canonical
transformation and stabilization of generalized Hamiltonian systems'',
\emph{Systems \& Control Letters}, vol. 42, no. 3, pp. 217-227, 2001.



\bibitem{Hogan}N. Hogan, ``Impedance control: An approach
to manipulation'' (Parts I, II, and III), \emph{ASME Journal Dynamic
Systems, Measurement, and Control}, vol. 107, no. 1, pp. 1-24, 1985.

\bibitem{HutchinsonHagerCorke1996}S.A. Hutchinson, G.D. Hager and
P.I. Corke, ``A tutorial on visual servo control'', \emph{IEEE Transactions
Robotics and Automation}, vol. 12, no. 5, pp. 651-670, 1996.

\bibitem{MaschkeArjan}B.M. Maschke and A.J. van der Schaft,
``Port-controlled Hamiltonian systems: modeling origins and system-theoretic
properties'', in \emph{Proc. IFAC Symp. on Nonlinear Control
Systems}, Bordeaux, France, 1992, pp. 282-288.

\bibitem{Munoz-Arias2013a}M. Munoz-Arias, J.M.A. Scherpen and D.A.
Dirksz, ``Position control via force feedback of a class of
standard mechanical system in the port-Hamiltonian framework'', in \emph{Proc. 52nd IEEE Conf. Decision and Control}, 2013, pp. 1622-1627.

\bibitem{Munoz-Arias2014b}M. Munoz-Arias, J.M.A. Scherpen and A. Macchelli, 
``An impedance grasping strategy'', in \emph{Proc.
53nd IEEE Conf. Decision and Control}, 2014, pp. 1403-1408.


\bibitem{Murray}R. Murray, Z. Li and S.S. Sastry, \emph{A
Mathematical Introduction to Robotic Manipulation}, CRC Press, 1994.

\bibitem{OrtegaLoriaNicklassonSira}R. Ortega, A. Loria, P.J.
Nicklasson and H. Sira-Ramirez, \emph{Passivity-Based Control
of Euler-Lagrange Systems}, New York, USA: Springer-Verlag, 1998.




\bibitem{Rijs} R. Rijs, R. Beekmans, S. Izmit and D. Bemelmans,
``Philips Experimental Robot Arm: User Instructor Manual'', Koninklijke Philips Electronics N.V., Eindhoven, vol 1.1, 2009.

\bibitem{Sadeghianetal2012}H. Sadeghian, M. Keshmiri, L. Villani and
B. Siciliano, ``Null-space impedance control with disturbance
observer'', in \emph{Proc. IEEE/RSJ Int. Conf.
Intelligent Robots and Systems}, Algarve, Portugal, 2012, pp. 2795-2800.

\bibitem{Sakai2012}S. Sakai and S. Stramigioli, ``Casimir
based impedance control'', in \emph{Proc. 2012 IEEE Int.
Conf. on Robotics and Automation}, Minnesota, USA, 2012, pp. 1384-1391.

\bibitem{Siciliano}B. Siciliano and O. Kathib, \emph{Springer
Handbook of Robotics}, Berlin, Germany: Springer, 2008.

\bibitem{Spongelal2006}M. Spong, S. Hutchinson and M. Vidyasagar, \emph{Robot modeling and control}, New York, USA: John Wiley and Sons,
Inc., 2006. 

\bibitem{Stramigioli2001}S. Stramigioli, \emph{Modeling
and IPC control of interactive mechanical systems: a coordinate-free
approach} (Lecture Notes in Control and Information Sciences
266), London, UK: Springer-Verlag, 2001.

\bibitem{Stramigioli2012}R. Mahony and S. Stramigioli, ``A
port-Hamiltonian approach to image-based visual servo control for
dynamic systems'', \emph{The International Journal of Robotic Research}, vol.
31, no. 11, pp. 1303-1319, 2012.

\bibitem{Arjan}A.J. van der Schaft, \emph{$L_{2}$-Gain
and Passivity Techniques in Nonlinear Control} (Lecture Notes in Control
and Information Sciences 218), London, UK: Springer-Verlag, 1999.

\bibitem{Viola}G. Viola, R. Ortega, R. Banavar, J. A. Acosta and
A. Astolfi, ``Total energy shaping control of mechanical systems:
simplifying the matching equations via coordinate changes'', \emph{IEEE
Transactions on Automatic Control}, vol. 52, no. 6, pp. 1093-1099, 2007.

\bibitem{angerer2017port}
M. Angerer, S. Musi{\'c}, and S. Hirche, ``Port-{H}amiltonian based control for human-robot team interaction'',\emph{2017 IEEE International Conference on Robotics and Automation (ICRA)},
pp 2292--2299, 2017.

\bibitem{rashad2019port},
R. Rashad, F. Califano, S. Stramigioli,
``Port-Hamiltonian Passivity-Based Control on SE (3) of a Fully Actuated UAV for Aerial Physical Interaction Near-Hovering'', 
IEEE Robotics and automation letters,
vol 4(4), pp 4378--4385, 2019.

\end{thebibliography}
\end{document}